\documentclass{article}
\usepackage[utf8]{inputenc}

\PassOptionsToPackage{sort}{natbib}

\usepackage{iclr2024_conference}
\usepackage{times}

\usepackage{amsmath}
\usepackage{amsthm}
\usepackage{amssymb}
\usepackage{booktabs}
\usepackage{comment}
\usepackage{mathtools}
\usepackage{microtype}
\usepackage{paralist}
\usepackage{siunitx}
\usepackage{subcaption}
\usepackage{tcolorbox}
\usepackage{wrapfig}
\usepackage{soul}

\usepackage[british]{babel}

\usepackage{tikz-cd}
\usepackage{shuffle}

\usepackage{algorithm}
\usepackage{algorithmicx}

\usepackage[noend]{algpseudocode}

\usepackage{hyperref}
\usepackage[nameinlink, noabbrev, capitalise]{cleveref}

\crefname{equation}{Eq.}{Eqs.}

\definecolor{bleu}{RGB}{ 49,140,231}

\hypersetup{
  colorlinks = true,
  linkcolor  = blue,
  filecolor  = blue,      
  citecolor  = blue,
  urlcolor   = blue,
}

\newcommand{\N}{\mathbb{N}}

\newcommand{\R}{\mathbb{R}}
\newcommand{\C}{\mathbb{C}}
\newcommand{\F}{\mathbb{F}}

\newcommand{\A}{\mathbb{A}}

\newcommand{\fg}{\mathfrak{g}}

\newcommand{\fF}{\mathfrak{F}}

\newcommand{\cS}{\mathcal{S}}

\newcommand{\Hom}{\mathsf{Hom}}

\newcommand{\D}{\partial}
\newcommand{\ep}{\varepsilon}

\newcommand{\mA}{\mathbbm{A}}

\newcommand{\mV}{\mathbbm{V}}
\newcommand{\mW}{\mathbbm{W}}

\newcommand{\J}{\mathbb{J}}

\newcommand{\W}{\mathbb{W}}

\newcommand{\V}{\mathbbm{V}}

\usepackage{amsthm}
\makeatletter
\newtheorem*{rep@theorem}{\rep@title}
\newcommand{\newreptheorem}[2]{\newenvironment{rep#1}[1]{\def\rep@title{#2 \ref{##1}}\begin{rep@theorem}}{\end{rep@theorem}}}
\makeatother

\newtheorem{theorem}{Theorem}
\newreptheorem{theorem}{Theorem}
\newtheorem{corollary}[theorem]{Corollary}
\theoremstyle{definition}
\newtheorem{definition}[theorem]{Definition}

\newtheorem{remark}[theorem]{Remark}
\newtheorem{lemma}[theorem]{Lemma}
\newreptheorem{lemma}{Lemma}

\newtheorem{proposition}[theorem]{Proposition}
\newreptheorem{proposition}{Proposition}

\newcommand\restr[2]{{\left.\kern-\nulldelimiterspace #1 \vphantom{\big|} \right|_{#2} }}

\title{Simplicial Representation Learning with\\Neural $k$-Forms} 
\iclrfinalcopy

\date{}
\author{Kelly Maggs$^{1}$, Celia Hacker$^{2}$, Bastian Rieck$^{3,4}$\\[+0.5cm]
$^1$École Polytechnique Fédérale de Lausanne (EPFL)\\
  $^2$Max Planck Institute for Mathematics in the Sciences\\
  $^3$AIDOS Lab, Institute of AI for Health, Helmholtz Munich\\
  $^4$Technical University of Munich~(TUM)
}

\begin{document}

\maketitle

\begin{abstract}
\emph{Geometric deep learning} extends deep learning to incorporate
information about the geometry and topology data, especially in complex
domains like graphs. Despite the popularity of message passing in this
field, it has limitations such as the need for graph rewiring, ambiguity
in interpreting data, and over-smoothing. In this paper, we take
a different approach, focusing on leveraging geometric information from
simplicial complexes embedded in $\R^n$ using node coordinates. We use
differential $k$-forms in $\R^n$ to create representations of simplices,
offering interpretability and geometric consistency without message
passing. This approach also enables us to apply differential geometry
tools and achieve universal approximation. Our method is efficient,
versatile, and applicable to various input complexes, including graphs,
simplicial complexes, and cell complexes. It outperforms existing
message passing neural networks in harnessing information from
geometrical graphs with node features serving as coordinates.
\end{abstract}

\section{Introduction}

\emph{Geometric deep learning}~\citep{Bronstein17a} expanded the scope
of deep learning methods to include information about the geometry--and,
to a lesser extent, topology---of data, thus enabling their use in more
complicated and richer domains like graphs. While the recent years have
seen the development of a plethora of methods, the predominant paradigm
of the field remains \emph{message passing}~\citep{Velickovic23a}, which
was recently extended to handle higher-order domains, including
\emph{simplicial complexes}~\citep{ebli2020simplicial}, \emph{cell
complexes}~\citep{Hajij2020Cell}, and
\emph{hypergraphs}~\citep{Heydari22a}.
However, despite its utility, the message passing paradigm suffers from
inherent limitations like over-smoothing, over-squashing, and an
inability to capture long-range dependencies.
These limitations often require strategies like graph rewiring, which
change the underlying graph structure~\citep{Gasteiger19a, Topping22a}
and thus affect generalisation performance.

Our paper pursues a completely different path and strives to leverage additional geometric
information from a data set to obtain robust and interpretable representations of the input data. Specifically, we consider input data in the form of simplicial complexes
embedded in $\R^n$ via node coordinates. This type of complex can be built from any graph with node features, with node features acting as the coordinates, for example. 
Our key insight is the use of \textit{differential $k$-forms} in $\R^n$.
A $k$-form in $\R^n$ can be integrated over any $k$-simplex embedded in
$\R^n$ to produce a real number.
Thus an $\ell$-tuple of $k$-forms produces an $\ell$-dimensional
\textit{representation} of the simplex independently of any message
passing.
From this perspective, $k$-forms play the role of globally consistent feature maps over the space of embedded $k$-simplices, possessing the geometric semantics and interpretability of integration.
This enables us to use tools from differential geometry to prove a version of universal approximation, as well as a number of other theoretical results.
Moreover, the structure of differential forms in $\R^n$ makes learning algorithms~(computationally) feasible.
In particular, a multi-layer perceptron with the right dimensions induces a $k$-form on $\R^n$ that can be integrated.
This implies that \textit{learnable}, finitely-parametrised differential forms can be woven into existing machine learning libraries and applied to common tasks in geometric deep learning.
We find that our method is better capable of harnessing information from
geometrical graphs than existing message passing neural networks.
Next to being \emph{efficient}, our method is also
\emph{generally applicable} to a wide variety of input complexes,
including graphs, simplicial complexes, and cell complexes.\footnote{For example, by taking barycentric subdivision, integration of forms over cell complexes is recoverable by integration over simplicial complexes.} 

\begin{tcolorbox}[
  boxsep     = 0.0mm,
  colframe   = bleu,
  colback    = bleu!10,
  left       = 2.0mm,
  right      = 2.0mm,
sharp corners,
]
  \textbf{In a nutshell:} We consider \textsc{data} in the form of
  simplicial chains on \emph{embedded simplicial complexes}, defining
  learnable \emph{differential $k$-forms} as \textsc{feature maps}, and
  introduce the concept of an \emph{integration matrix}, which serves as
  an overall \textsc{representation} of the data.
\end{tcolorbox}

\paragraph{Organisation of Paper.}
We present the relevant background for simplicial complexes and
differential forms in \Cref{sec:background}. In \Cref{sec:neural_forms},
we introduce neural $k$-forms and prove a universal
approximation statement. We also show how neural $k$-forms induce
a so-called integration matrix, and use the properties of integration to
prove a number of propositions. In \Cref{sec:architecture} we present the basic architecture and algorithms. Finally, in \Cref{sec:experiments}, we
perform several intuitive experiments and benchmark our method on
standard geometrical deep learning data sets.

\paragraph{Related Work.}
Several methods focus on generalising
graph neural networks~(GNNs) to higher-dimensional domains, proposing 
\emph{simplicial neural networks}~\citep{ebli2020simplicial,
Bodnar2021WeisfeilerAL, Bunch2020Simplicial, Roddenberry2021principled,
Keros2021dist,Giusti2022simplicial,Goh2022simplicial},
methods that can leverage higher-order topological features of data~\citep{hajij2023topological, hensel2021survey, horn2022topological}, or optimisation algorithms for learning representations of simplicial complexes~\citep{hacker2021ksimplex2vec}.
All of these methods operate on simplicial complexes via \emph{diffusion
processes} or \emph{message passing} algorithms.
Some works also extend message
passing or aggregation schemes to \emph{cell complexes}~\citep{Bodnar2021CW,Hajij2020Cell, hajij2023topological} or \emph{cellular sheaves}~\citep{han19sheaf, Barbero2022Sheaf}. 
However, these existing methods exhibit limitations arising from the use
of message passing or aggregation along a combinatorial structure.
Message passing often results in \emph{over-smoothing}~(a regression to the
mean for all node features, making them indistinguishable) or
\emph{over-squashing}~(an inability to transport long-range node
information throughout the graph), necessitating additional
interventions~\citep{Gasteiger19a, Nguyen23a, Topping22a}.
Hence, there is a need for methods that go beyond message passing.
Our work provides a novel perspective based on the integration of
learnable $k$-forms on combinatorial domains like graphs or simplicial
complexes embedded in $\R^n$, i.e.\ we assume the existence
of~(vertex) coordinates.

\section{Background}\label{sec:background}

This section introduces the required background of simplicial complexes
and differential forms. We restrict our focus to simplices and
differential forms in $\R^n$, given this is the only setting we will use
in practice to make the theory more accessible. For additional
background references, we recommend \citet{Nanda21} for computational
topology and \citet{Lee03} or \citet{Tao_2009} for differential forms. More details are also provided in Appendix \ref{diff_form_appendix}. 

\paragraph{Abstract Simplicial Complexes.}
An abstract simplicial complex $\mathcal{S}$ is generalisation of
a graph on a set of vertices $\mathcal{S}^0 $ . In a graph, we have
pairwise connections between nodes given by edges, or $1$-dimensional
simplices~(denoted by $\mathcal{S}^1$). In simplicial complexes, there
are higher-dimensional counterparts of these connections, called the
$k$-dimensional simplices, or just $k$-simplices. The $k$-simplices are
connections between $k+1$ vertices of the set $\mathcal{S}^0$, with
a $2$-simplex forming a triangle, a $3$-simplex forming a tetrahedron,
and so on. We denote a $k$-simplex by $\sigma = [v_0, \dots, v_k ]$,
where $v_i\in \mathcal{S}^0$, writing $\mathcal{S}^k$ to refer to the
set of all $k$-simplices. If a simplex $\sigma \in \mathcal{S}$, then so
are all of its faces, i.e.\ all simplices formed by subsets of
the vertices of $\sigma$.

\paragraph{Affine Embeddings.} Data in geometric deep learning most
often comes as a geometric object---a graph or simplicial
complex---combined with node features in $\R^n$. Formally, node features
correspond to a \emph{node embedding} of a simplicial complex $\cS$,
which is a map $\phi\colon \cS^0 \to \R^n$. The \emph{standard geometric
$k$-simplex} is the convex hull $\Delta_k = [0,t_1, \cdots, t_k] \subset
\R^k$, where $t_i$ is the endpoint of the $i$-th basis vector. For each
$k$-simplex $\sigma = [v_0, \ldots, v_k]$ the map $\phi$ induces an
\emph{affine embedding} $\phi_\sigma\colon \Delta^k \to \R^n$ whose image
is the convex hull $[\phi(v_0), \ldots, \phi(v_k)]$.

\paragraph{Chains and Cochains over $\R$.}
For an oriented\footnote{The choice of orientation of each simplex corresponds to a choice of
  sign for each basis vector.
}
simplicial complex $\mathcal{S}$, the \emph{simplicial $k$-chains}
$C_k(\mathcal{S};\R)$ are the vector space
\begin{equation}
    C_k(\mathcal{S};\R) = \Bigg\{ \sum_i \lambda_i \sigma_i \mid \sigma_i \in \cS^k, \lambda_i \in \R \Bigg\}
\end{equation}
of formal linear combinations of $k$-simplices in $\cS$. The
\emph{simplicial $k$-cochains} $C^k(\mathcal{S};\R)$ over $\R$ are the
dual space $\Hom(C_k(\mathcal{S}); \R)$ of linear functionals over the
simplicial $k$-chains.

\paragraph{Differential Forms.} The \emph{tangent space }$T_p(\R^n)$ at $p$ is the space of all vectors originating at a point $p \in
\R^n$, and its
elements are called \emph{tangent vectors}.
In our case, this is space is canonically isomorphic to the underlying space $\R^n$.
A differential form is a function that assigns a notion of \emph{volume} to
tuples of tangent vectors at each point in $\R^n$. Given a tuple of $k$ standard basis vectors $(e_{i_1},\cdots, e_{i_k})$
in $\R^n$ indexed by $I = (i_1, i_2, \ldots, i_k)$ and a vector $v \in
T_p(\R^n)$, the vector $v^I$ is the projection of $v$ onto the
$I$-subspace spanned by $(e_{i_1}, e_{i_2}, \dotsc, e_{i_k})$. The
associated \emph{monomial $k$-form} 
\begin{equation}
  dx_I(v_1,  v_2, \ldots, v_k ) = \varepsilon^I(v_1,  v_2, \ldots, v_k ) := \det \big[ v^I_1,  v^I_2, \ldots, v^I_k \big]
\end{equation}
represents the standard volume spanned by $k$ tangent vectors $v_i \in
T_p(\R^n)$ in the $I$-subspace of the tangent space. 

\paragraph{Scaling Functions.} General differential forms are built from
locally linear combinations of re-scaled monomial $k$-forms. Formally,
a general differential $k$-form $\omega \in \Omega^k(\R^n)$ can be
written as
\begin{equation}
  \label{form_decomp}
  \omega_p(v_1,v_2, \ldots, v_k) = \sum_I \alpha_I(p) dx_I(v_1,v_2, \ldots, v_k),
\end{equation}
where $p \in \R^n$ and $v_i \in T_p(\R^n)$ and $I$ ranges over the
$\binom{n}{k}$ subspaces spanned by sets of $k$ basis vectors in $\R^n$.
The \emph{scaling functions} $\alpha_I\colon \R^n \to \R$ represent
a re-scaling of the standard volume in the $I$-subspace for each point
$p \in \R^n$.
Intuitively, the scaling functions in a differential $k$-form specify
the size and orientation of a subspace for each point. 

\paragraph{Integration.}
Differential $k$-forms can be integrated over embedded $k$-simplices
$\phi\colon\Delta^k \to \R^n$. Let $D\phi(t)$ be the Jacobian matrix of
$\phi$ at $t \in \Delta^k$.
For affinely-embedded simplices,  $D\phi(t)$ is given by
\begin{equation} 
  D\phi_{i,j} = \Big[ \phi^i(t_j) - \phi^i(0) \Big]_{i,j}.
\end{equation}
The integral of $\omega \in \Omega^k(\R^n)$ over the image of $\phi$ can be expressed as
\begin{equation}
  \label{integration_eq}
  \int_\phi \omega := \sum_I \int_{\Delta^k} \alpha_I(\phi(t)) \varepsilon^I\big(D\phi \big) dt,
\end{equation}
where the integral is interpreted as a standard Riemann integral over
the~(compact) subset $\Delta^k \subseteq \R^k$. This represents the
signed volume of the image of $\phi$ with respect to the differential
form $\omega$. In practice, this integral is approximated with a finite
Riemann sum. We provide more background on integration of forms in
\Cref{diff_form_appendix}. 
The integral in \Cref{integration_eq} is well-defined if
$\phi$ is a $C^1$ embedding function and $\alpha_I \circ \phi$ is
integrable. 

\section{Neural \texorpdfstring{$k$}{k}-Forms and Integration Matrices} \label{sec:neural_forms}

Broadly speaking, representation learning is the process of turning data
into vectors in $\R^\ell$, followed by the use of standard tools of
machine learning to classify/predict attributes based on these vector
representations. In graph learning and simplicial learning more
generally, data comes in the form of simplicial complexes embedded in
$\R^n$. 
One simple vectorization is to take the standard \textit{volume} of
a $k$-simplex embedded in $\R^n$.
\textbf{The central idea of our paper is to
create vector representations of $k$-simplices as volumes relative to
a tuple of differential $k$-forms.}
Indeed, a tuple of $k$-forms
$\omega_1, \omega_2, \dotsc, \omega_\ell \in \Omega^k(\R^n)$ determines
a representation function
\begin{equation}
    (\phi\colon\Delta^k \to \R^n) \mapsto \Big( \int_\phi \omega_1, \dotsc, \int_\phi \omega_\ell \Big) \in \R^\ell
\end{equation}
that vectorises any $k$-simplex embedded in $\R^n$ by calculating its
volume via integration of forms. 
In this paradigm, each $k$-form takes the role of a \emph{feature map}
on the space of embedded $k$-simplices.

\paragraph{Neural $k$-forms.}
The key to making integral-based representations \textit{learnable} is
to model the scaling functions in \cref{form_decomp} as a multi-layer
perceptron~(MLP).
For a set of $k$ basis vectors $I$, let $\pi_I \colon \R^{\binom{n}{k}}
\to \R$ be the \emph{projection onto the $I$-subspace} and define
$\psi_I := \pi_I \circ \psi \colon \R^{n} \to \R$.
\begin{definition}
  \label{def:neural_kform}
  Let $\psi\colon\R^n \to \R^{\binom{n}{k}}$ be an MLP.
The \emph{neural $k$-form} $\omega^\psi \in \Omega^k(\R^n)$
  associated to $\psi$ is the $k$-form $\omega^\psi = \sum_I \psi_I dx_I$.
\end{definition}
In words, the components of the MLP correspond to the scaling functions
of the neural $k$-form. \textbf{The goal of a neural $k$-form is thus to learn
the size and orientation of $k$-dimensional subspaces at each point in
$\R^n$ according to some downstream learning task.}
In case the activation function is a sigmoid function or $\tanh$,
\cref{def:neural_kform} produces smooth $k$-forms, whereas for
a \texttt{ReLU} activation function, one obtains piecewise linear
$k$-forms.

\begin{remark}
A neural $0$-form over $\R^n$ with $\ell$ features is an MLP $\psi$ from
$\R^n$ to $\R^\ell$. The $0$-simplices $p \in \R^n$ are points in
$\R^n$, and integration of a $0$-form corresponds to the evaluation
$\psi(p)$. In this way, neural $k$-forms are a direct extension of MLPs
to higher-dimensional simplices in $\R^n$.
\end{remark}

\paragraph{Universal Approximation.}
We would like to know which $k$-forms on $\R^n$ are possible to
approximate with neural $k$-forms. The following proposition translates
the well-known Universal Approximation Theorem~\citep{Cybenko_1989,
Hornik_Stinchcombe_White_1989} for neural networks into the language of
neural $k$-forms. Here, the norm $\lVert \, \, \,
\rVert_{\Omega_c(\R^n)}$ on $k$-forms is induced by the standard
Riemannian structure on $\R^n$, as explained in \Cref{proof-appendix}.

\begin{theorem}\label{prop:univ-approx-1}
    Let $\alpha \in C(\R, \R)$ be a non-polynomial activation function.
    For every $n \in \N$ and compactly supported $k$-form $\eta \in
    \Omega^k_c(\R^n)$ and $\epsilon > 0$ there exists a neural $k$-form
    $\omega^\psi$ with one hidden layer such that $ \lVert \omega^\psi
    - \eta \rVert_{\Omega_c(\R^n)} < \epsilon$.
\end{theorem}

\begin{wrapfigure}[17]{l}{0.50\linewidth}
        \centering
\includegraphics[width=\linewidth]{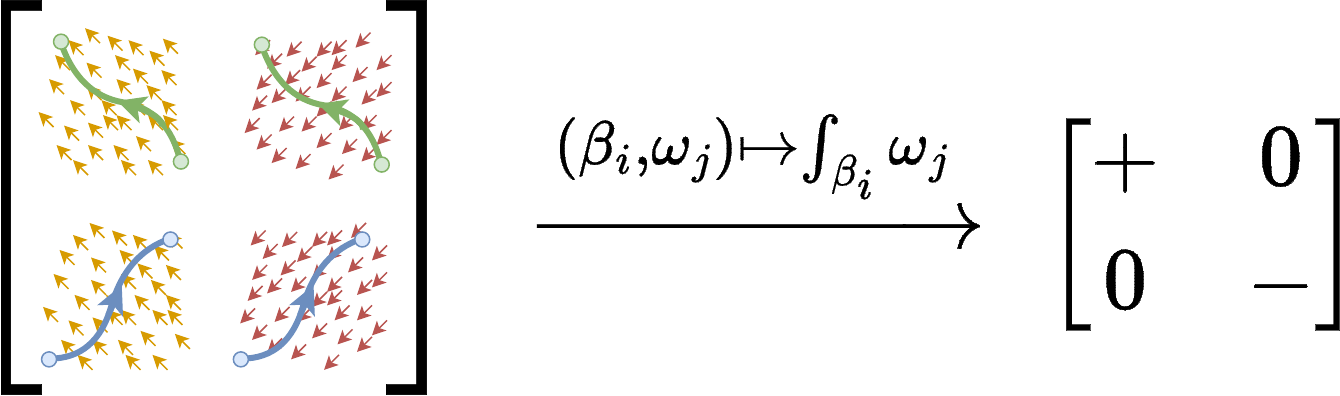}
        \caption{An \emph{integration matrix} with data in dimension~$1$,
        where embedded oriented $1$-simplices correspond to paths and $1$-forms are canonically identified with vector fields~(see \cref{diff_form_appendix} for details). Integration of a $1$-form against a path corresponds to path integration against the respective vector field. Thus, integration of  the paired paths and $1$-forms in the left matrix recovers real values with the signs given in the right matrix.
        }
    \label{fig:integration_matrix}
\end{wrapfigure}
\paragraph{Cochain Matrices.}
Most simplicial neural networks follow a similar procedure.
The key step is that the $k$-simplices of each simplicial complex $\cS$
are replaced by a matrix $X_\cS(\beta,\gamma)$ containing a selection
of $\ell$ simplicial cochains $(\gamma_1, \ldots, \gamma_\ell)$ as
columns. This can be thought of as a matrix containing evaluations
$\gamma_j(\beta_i)$ with respect to some basis $(\beta_1, \ldots,
\beta_s)$ for the simplicial $k$-chains $C_k(\cS;\R)$.
However, there is no canonical
initialisation when one starts with only the data of a set of embedded
finite simplicial complexes.\footnote{Indeed, if one takes a random initialisation, the feature cochains do
not have a shared interpretable meaning across different complexes.
}
We address this issue by introducing the \textit{integration matrix}
induced by a neural $k$-form. This produces the same data type, but has
the advantage that the feature cochains correspond to integration of the
\emph{same form} defined over the ambient space. 

\paragraph{Integration Matrices.}
For an embedded finite simplicial complex, an $\ell$-tuple $\omega
= (\omega_1, \omega_2, \ldots, \omega_\ell)$ of $k$-forms induces
a matrix suitable to simplicial learning algorithms in a natural way via
integration. Let $\phi\colon \mathcal{S} \to \R^n$ be an affine embedding of
a simplicial complex and $\beta = (\beta_1, \beta_2, \ldots, \beta_m)
\in C_k(\mathcal{S};\R)$  be a set of specified $k$-chains.
\begin{definition} 
The \emph{integration matrix} with respect to $\beta$ and $\omega$ is $X_\phi(\beta, \omega) = \big[ \int_{\beta_i} \omega_j \big]_{i,j}.$
\end{definition}
The integral of $\omega$ over a simplicial chain $\beta = \sum
\lambda_\sigma \sigma \in C_k(\cS;\R)$ with respect to the embedding
$\phi\colon \cS \to \R^n$ corresponds to the integral
\begin{equation}
\label{int-embedded-simplicial}
\int_{\beta} \omega = \sum \lambda_\sigma \int_{\phi_{\sigma}} \omega.
\end{equation} 

\begin{remark}[Interpretability]
The key point is that the integration matrices of two different
  simplicial complexes $\phi\colon \cS \to \R^n$ and $\phi'\colon\cS'
  \to \R^n$ embedded in $\R^n$ have a shared interpretation.
That is, the $j$-th column of both of their integration matrices
  corresponds to the volume of $k$-simplices against \emph{the same
  feature $k$-form} $\omega_j \in \Omega^k(\R^n)$.
\end{remark}
A tuple of $\ell$ neural networks $\psi_j \colon \R^n \to \R^{\binom{n}{k}}$
with associated neural $k$-forms $\omega^{\psi_j}$ induces
a \emph{learnable} matrix representation
\begin{equation}
  (\beta, \psi) \mapsto X_\phi(\beta, \omega^\psi) = \Big[ \int_{\beta_i} \omega^{\psi_j} \Big]
\end{equation}
of the given simplicial chain data $\beta$ with embedding $\phi\colon \cS
\to \R^n$. This intermediate representation is finitely parametrised by
$\psi$ and can thus be updated via backpropagation.
The matrix representation of a simplex---and by extension a simplicial
chain---depends on whether the neural $k$-form decides that it is
embedded in a large or small subspace, and with what orientation.

\paragraph{Basic Properties.}
There are a number of useful basic properties about integration matrices
that follow from the well-known properties of integration. In the next
proposition, we conceptualise
\begin{equation} 
\beta = (\beta_1, \ldots, \beta_m)^T \in M^{m \times 1}\Big(C_k(\cS;\R)\Big) \text{ and } \omega = (\omega_1, \ldots, \omega_\ell) \in M^{1\times \ell}\Big(\Omega^k(\R^n)\Big)
\end{equation}
as chain-valued column and $k$-form valued row vectors, respectively.
Real-valued matrices act on both vectors by scalar multiplication and
linearity.
\begin{proposition}[Multi-linearity] \label{multi-linearity}
    Let $\phi\colon\mathcal{S} \to \R^n$ be an embedded simplicial
    complex.
For any matrices $L \in M^{m' \times m}(\R)$ and $R \in M^{\ell
    \times \ell'}(\R)$, we have
\begin{equation}
        X_\phi(L \beta, \omega R) = L X_\phi(\beta, \omega) R
    \end{equation}
\end{proposition}
A staple requirement of geometric deep learning architectures is that
they should be permutation and orientation equivariant. In our setting,
these properties are a direct corollary of \Cref{multi-linearity}.
\begin{corollary}[Equivariance]
Let $\beta = (\beta_1, \ldots, \beta_m)$ be a basis for the $k$-chains
  $C_k(\cS;\R)$ of an embedded oriented simplicial complex $\phi\colon\cS
  \to \R^n$.
\begin{enumerate}
        \item (Permutation) $X_\phi(P \beta,\omega)
          = P X_\phi(\beta,\omega)$ for all permutation matrices $P \in
          M^{m \times m}(\R)$.
        \item (Orientation) $ X_\phi(Q \beta, \omega)
          = Q X_\phi(\beta,\omega)$ for all signature
          matrices\footnote{A signature matrix is a diagonal matrix with $\pm 1$ entries.
          } $Q \in M^{m \times m}(\R)$.
    \end{enumerate}
\end{corollary}

\section{Architecture} \label{sec:architecture}

\paragraph{Embedded Chain Data.}
The input data $\mathcal{D} = \{ (\cS_\alpha, \phi_\alpha, \beta_\alpha)
\}$ to our learning pipeline consists of a set of triples $(\cS_\alpha,
\phi_\alpha, \beta_\alpha)$ of \emph{embedded chain data}, where
\begin{itemize} 
\item $\mathcal{S}_\alpha$ is a simplicial complex
\item $\phi_\alpha \colon \mathcal{S}_\alpha^0 \to \R^n$ is an affine embedding and
\item $\beta_\alpha \in \bigoplus_{m_\alpha} C_k(\mathcal{S};\R)$ is a tuple of $m_\alpha$ data $k$-chains on $\mathcal{S}_{\alpha}$.
\end{itemize}
If no chains are provided, one can take the standard basis of oriented
$k$-simplices of each simplicial complex as the input chains. The
canonical example is a dataset consisting of graphs $\{
(\mathcal{G}_\alpha, \phi_\alpha, \beta_\alpha) \}$ with node features
$\phi_\alpha \colon \mathcal{G}^0_\alpha \to \R^n$ and $\beta_\alpha$
corresponding to the standard edge chains with arbitrary orientations.

\paragraph{Approximating Integration Matrices.}
The main departure from standard geometric deep learning architectures
is the transformation from embedded $k$-chain data to integration
matrices. The high-level structure of this process is presented in
\Cref{alg:int_matrix}. Given a tuple of neural $k$-forms, each
represented by an MLP $\psi_j \colon \R^n \to \R^{\binom{n}{k}}$, integrals
of embedded $k$-simplices $\phi_\sigma \colon \Delta^k \to \R^n$ are
calculated by a finite approximation~(\cref{sec:A-integration})
of the integral formula 
\begin{equation}
  \int_{\phi_\sigma} \omega_j = \sum_I \int_{\Delta^k} \psi_{I,j} (\phi_{\sigma}(t)) \varepsilon^I\big(D \phi_{\sigma} (t)\big) dt \approx \texttt{VolApprox}( \int_{\phi_{\sigma}} \omega_j)
\end{equation}
appearing in \Cref{integration_eq}. Integrals of chains $\beta_i$ are
calculated as linear combinations. \textbf{The entries of the integration matrix
$X_\phi(\beta,\omega^\psi)$ thus depend in a differentiable manner on
the component functions $\psi_{I,j}$ of the underlying MLP}.

\paragraph{Readout Layers.}
Once an embedded simplicial complex is transformed into an integration
matrix, it is then fed into a \emph{readout layer}~(as in the case for standard
simplicial or graph neural networks). The output of a readout layer is
a single representation of the entire complex, and should not depend on
the number of simplices if one wishes to compare representations among
different complexes. Common read-out layers include summing column
entries, and $L_1$ or $L_2$ norms of the columns. Note that only the
latter two are invariant under change of orientation. 

\paragraph{Neural $k$-form Backpropagation.} For a fixed dataset of
embedded $k$-chains, neural $k$-forms are the learnable feature
functions. \Cref{alg:kform_backprop} and \Cref{fig:architecture}
illustrate the basic pipeline for performing backpropagation of neural
$k$-forms over embedded chain data and a loss function. Neural $k$-forms
and a user-determined classifier are randomly initialised using
\emph{any} standard MLP initialisation. Each embedded chain data
$(\mathcal{S}_\alpha, \phi_\alpha, \beta_\alpha)$ is transformed into an
integration matrix by \Cref{alg:int_matrix}, before being read out,
classified, and evaluated against a loss function. 

\begin{figure}[tbp]
    \centering
    \includegraphics[width=\linewidth]{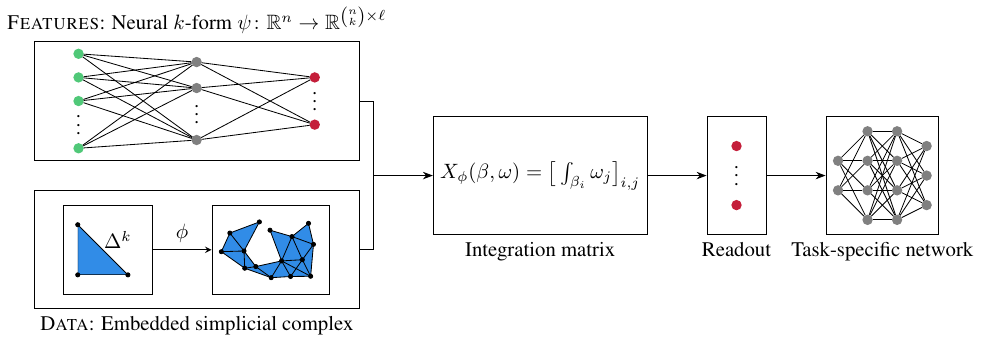}
    \caption{
    A schematic of our proposed neural $k$-form learning architecture.}
\label{fig:architecture}
\end{figure}

\paragraph{Implementation.}
Our methods can be realised using any deep learning framework that permits training an MLP.
We created a proof-of-concept implementation using \texttt{PyTorch-Geometric}~\citep{PyG} and \texttt{PyTorch-Lightning}~\citep{Torch_Lightning} 
and make it publicly available under
\url{https://github.com/aidos-lab/neural-k-forms}.

\begin{algorithm}[t]
      \caption{Generate Integration Matrix}
      \label{alg:int_matrix}
\begin{algorithmic}
\State \textbf{Inputs:} Embedded chain data $(\mathcal{S},\phi, \beta)$
with $k$-simplices $\sigma \in \cS^k$;\\ $k$-chains $\beta_i = \sum
\lambda_{\sigma,i} \sigma \in C_k(\cS;\R)$; $i = 1, \ldots, m$; \\
MLPs\footnote{Equivalently, a single MLP $\psi\colon \R^n \to
\R^{\binom{n}{k} \times \ell}$ where $\psi_j = \pi_{I_j} \psi$}
$\psi_j\colon \R^n \to \R^{\binom{n}{k}}$; $j = 1,\dots,
\ell$.

\noindent\hrulefill
\State $X = \mathbf{0} \in M_{m,\ell}(\R)$ \Comment{Initialise $X$ as the $0$-matrix with $m \times \ell $ elements }
\For {$ 1 \leq j \leq \ell$} \Comment{Iterate over forms}
    \For {$1 \leq i \leq m$} \Comment{Iterate over chains}
        \State $X_{i,j} \gets \sum \lambda_{\sigma,i} \texttt{VolApprox}(\int_{\sigma} \omega^{\psi_j})$ 
  \EndFor
\EndFor
  \State \textbf{Return:} $X_\phi(\beta,\omega^\psi) = X$ \Comment{Return Integration Matrix}
\end{algorithmic}

\end{algorithm}

\begin{algorithm}[t]
      \caption{Neural $k$-form Backpropagation}
      \label{alg:kform_backprop}
\begin{algorithmic}
\State \textbf{Data:} $\mathcal{D} = \{ ( \cS_\alpha, \phi_\alpha,
\beta_{\alpha} \big) \mid \phi_\alpha\colon \cS_{\alpha} \to \R^n, \beta_{\alpha} \in \bigoplus_{m_\alpha} C_k(\cS;\R), m_\alpha \in \N \}$ 
\State \textbf{Initialise} MLPs $\psi_j\colon \R^n \to
\R^{\binom{n}{k}}; 1 \leq j \leq \ell$; classifier $\eta\colon \R^\ell \to \R^{\ell'}.$

\noindent\hrulefill
\For {$(\mathcal{S}_\alpha, \phi_\alpha, \beta_\alpha) \in \mathcal{D}$}
\State $X_{\phi_\alpha}(\beta_\alpha,\omega^\psi) \gets \texttt{IntegrationMatrix}(\phi_\alpha, \beta_\alpha,\psi)$
\State $X \gets \texttt{Readout}(X_{\phi_\alpha}(\beta_\alpha,\omega^\psi))$ \Comment{Vector Representation}
\State $X \gets \eta(X)$ \Comment{Prediction}
\State $L = \texttt{Loss}(X)$
\State $\texttt{Backward}(L,\psi, \eta)$ \Comment{Update $k$-forms and Classifier}
\EndFor
\end{algorithmic}

\end{algorithm}

\section{Experiments and Examples}
\label{sec:experiments}

This section presents examples and use cases of neural $k$-forms in
classification tasks, highlighting their \emph{interpretability} and
\emph{computational efficiency}.

\subsection{Synthetic Path Classification}

Our first experiment is classifying paths in $\R^2$. The goal of the
experiment is pedagogical; it illustrates how to interpret the learned
$1$-forms rather than benchmark the method. A piecewise linear path in
$\R^n$ is a simplicial complex determined by an ordered sequence of node
embeddings. The $1$-simplices are linear maps $\sigma_i\colon I \to
\R^n$ from the $i$-th embedded node and to the $(i+1)$-st embedded node,
where the full path corresponds to the $1$-chain $\sum_i \sigma_i$. The
integral of this chain against a 1-form corresponds to the path
integral. \cref{fig:final-vf1} shows three classes of piecewise linear
paths\footnote{The orientation is indicated by the arrow at the end of
each path.} in $\R^2$ that we will classify using our method.
The idea is that we will learn three $1$-forms, which correspond to each
class. To initialise the three $1$-forms in $\R^2$, we randomly
initialise an MLP $\psi\colon \R^2 \to \R^{2\times 3}$. A forward pass
consists of two stages: \emph{integration} of the $1$-simplices in the path
against each of the three forms to produce an integration matrix,
followed by taking a \emph{column sum} and applying \texttt{softmax} to produce
a distribution over which we perform \texttt{CrossEntropyLoss}.
The $i$-th column sum corresponds to a path integral of the path against
the $i$-th form; the prediction of a path is thus determined by which
feature $1$-form produces the highest path integral. Backpropagation
against this loss function thus attempts to modify the $i$-th $1$-form
so that it produces a more positive path integral against the paths in
class $i$ and more negative otherwise. 

\cref{fig:final-vf1} also shows the feature $1$-forms as
vector fields over their corresponding classes of paths. Note that the vector
fields are randomly initialised and updated while the paths are the
fixed data points. The learned $1$-forms of each class resemble vector
fields that roughly reproduce the paths in their class as integral flow
lines, and are locally orthogonal or negatively correlated to paths in
the other classes. This is a direct result of the objective function,
which attempts to maximise the path integral within each class and
minimise it for others. \cref{fig:final-embeddings1} depicts the
paths as points coloured by class, where the coordinates correspond to
the path integrals against the three learned $1$-forms. We observe a clear
separation between the classes, indicating that the representations
trivialise the downstream classification task. We also compare the model
with a standard MLP in \Cref{app:exp-details}.

\begin{figure}
\centering
\subcaptionbox{Learned $1$-forms corresponding to paths in each class.\label{fig:final-vf1}}{
  \includegraphics[width=0.75\textwidth]{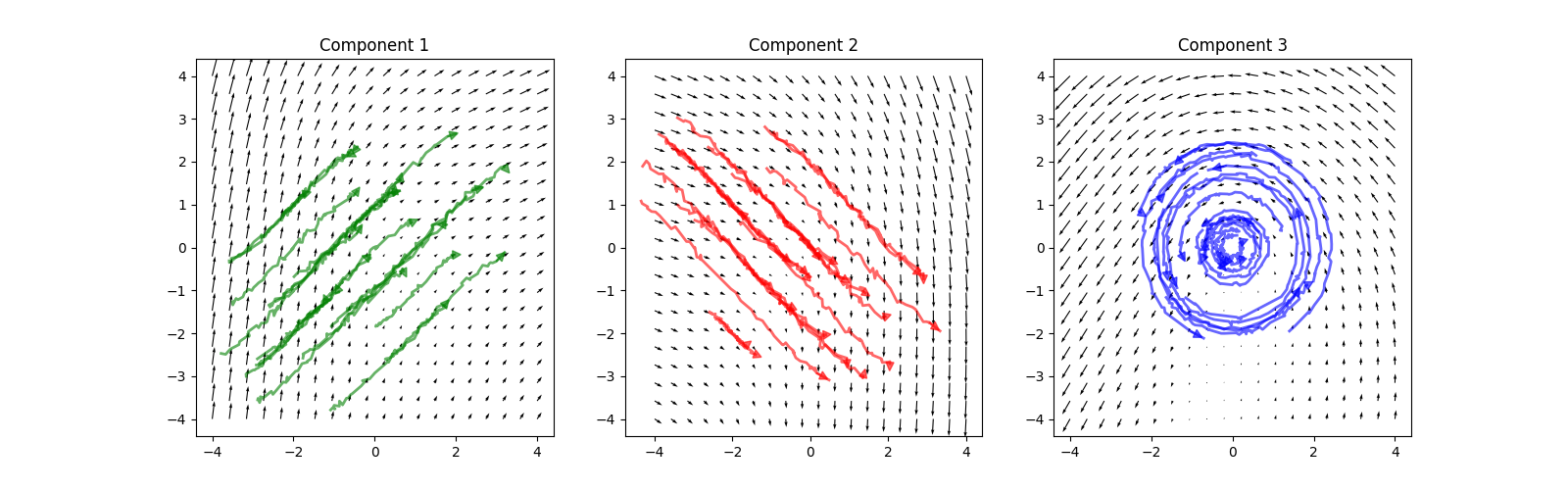}}\subcaptionbox{Path representations.\label{fig:final-embeddings1}
}{\includegraphics[width=0.25\textwidth]{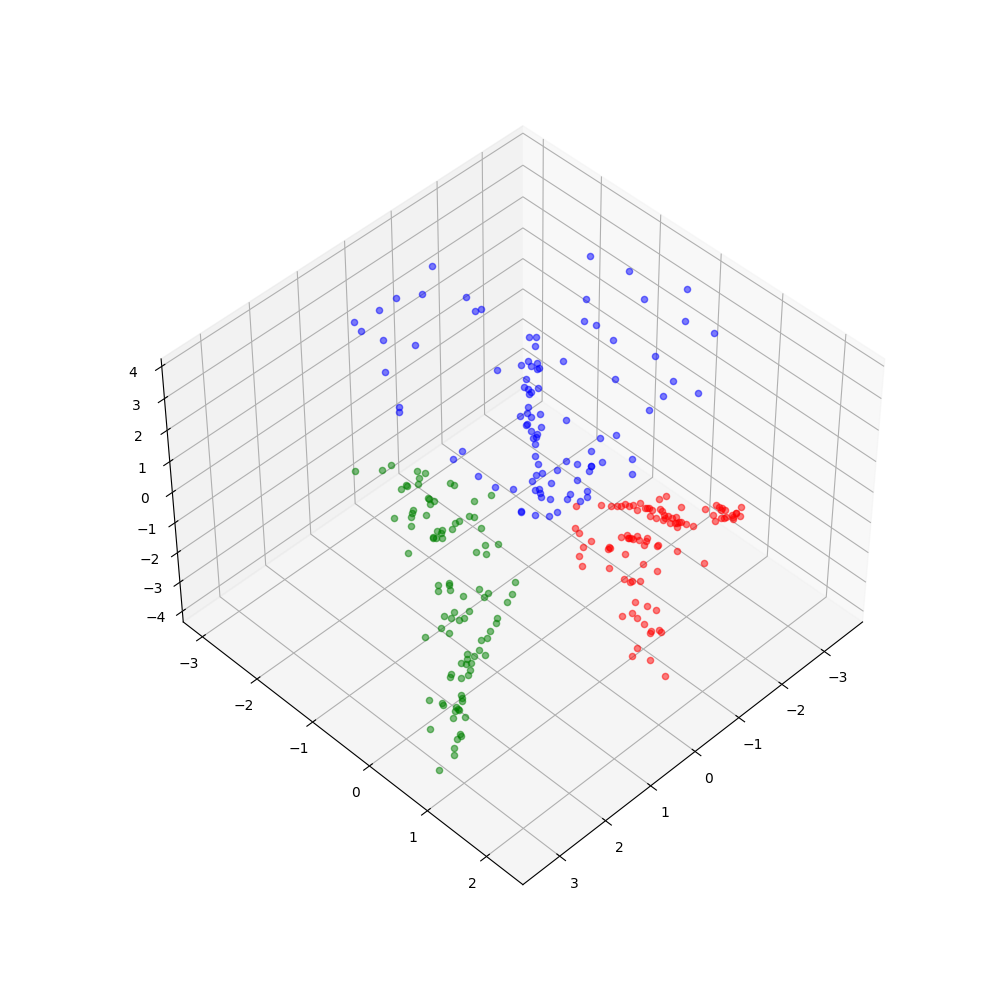}
}
\caption{Synthetic Path Classification via Learnable $1$-forms. The $1$-form~(a
  vector field in this case) adjusts itself to the data, resulting in
  distinct path representations.
}
\label{fig:synth-path-ex}
\end{figure}

\subsection{Synthetic Surface Classification}
This example demonstrates how our framework is of interest to deal
with higher-dimensional data, i.e.\ simplicial complexes of dimension
$k\geq2$. Conceptually, the difficulty is going from 1-dimensional
objects to $k$-dimensional objects with any $k\geq 2$. We restrict
ourselves to $k=2$ since higher dimensions are similar to this case.
The data we consider here are triangulated surfaces embedded in $\R^3$,
the underlying combinatorial complex is always the same---
a triangulation of a square---but the embeddings are different.
For a given surface, the embedding of each  $2$-simplex is given by
linear interpolation between the coordinates of its three vertices in
$\R^3$. We consider two classes of surfaces, in the first class the
embeddings of the nodes are obtained by sampling along a sinusoidal
surface in the $x$-direction, while the second class is given by a surface
following a sinusoidal shape in the $y$-direction, in each case with
added translation and noise. 
We use learnable $2$-forms $\omega_1, \omega_2$ on $\R^3$ represented by
an MLP $\psi\colon\R^3 \to \R^{3\times 2}$. In a forward pass of the
model, we integrate the $2$-forms given by the MLP over the $2$-simplices of
a surface. Each point in $\R^3$ has a value in $\R^{3\times 2}$ given by
the MLP evaluated at that point. The process of integrating over the
$2$-simplices corresponds to integrating the point-wise evaluation of
the MLP over the regions of $\R^3$ defined by each embedded $2$-simplex.
This process gives the integration matrix of the forms $\omega_1,
\omega_2$ over the $2$-simplices of the complex. 

\setlength\intextsep{0pt}
\begin{wrapfigure}[13]{l}{0.35\linewidth}   
  \centering
  \includegraphics[width=\linewidth]{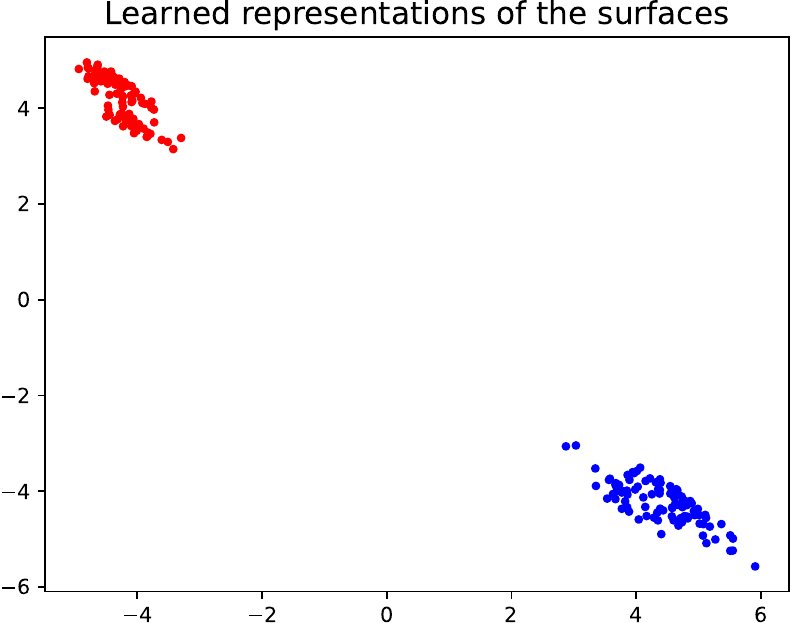}
  \caption{Synthetic surface classification via learnable $2$-forms.
  }
  \label{fig:final-surface-embeddings} 
\end{wrapfigure}
The next step is the readout layer, which sums the entries in each
column of the integration matrix, corresponding to summing the value of
each cochain over the simplicial complex, giving the total surface
integral. As there are two $2$-forms this yields a vector representation
in $\R^2$, each entry corresponding to a $2$-form. 
Finally, these representations are then passed through
a \texttt{softmax} for classification and \texttt{CrossEntropyLoss} is
used as a loss function. The MLP is then updated by backpropagation. 
\cref{fig:final-surface-embeddings} plots the representations in
$\R^2$ learned through the model described above. Each point represents
a surface in the data set and the colour is given by the class of the
corresponding surface, showing the clear separation learned by the
neural $2$-forms. Further details can be found in \Cref{sec:surface_ap}. 

\begin{table}[tbp]
  \centering
  \caption{Results~(mean accuracy and standard deviation of a $5$-fold
    cross-validation) on small graph benchmark datasets that exhibit
    `geometrical' node features.
Parameter numbers are approximate because the number of classes 
    differ.
   }
  \label{tab:Benchmark datasets small}
  \sisetup{
    detect-all              = true,
    table-format            = 2.2(2),
    separate-uncertainty    = true,
    retain-zero-uncertainty = true,
  }\let\b\bfseries \newcommand{\B}{}\resizebox{\linewidth}{!}{\begin{tabular}{lrSSSSSS}
      \toprule
                                         & {Params.} & {BZR}          &   {COX2}       & {DHFR}         & {Letter-low}   & {Letter-med}   & {Letter-high}  \\
      \midrule
      \B EGNN~\citep{Satorras21a}        & 1M        & \B79.51\pm1.87 & \B78.16\pm0.46 &\B 64.02\pm2.68 & \B93.20\pm0.68 & \B65.91\pm1.60 & \b\B68.98\pm 1.80\\
      GAT~\citep{Velickovic22a}          & 5K        & \b81.48\pm2.90 &   80.73\pm2.45 & \b  73.02\pm2.54 &   90.04\pm2.23 &   63.69\pm5.97 &     43.73\pm 4.13\\
      GCN~\citep{kipf2017semisupervised} & 5K        &  79.75\pm0.68 & \b79.88\pm1.65 &   70.12\pm5.43 &   81.38\pm1.57 &   62.00\pm2.07 &     43.06\pm 1.67\\
      GIN~\citep{xu2018how}              & 9K        &   79.26\pm1.03 &   78.38\pm0.79 &   68.52\pm7.38 &   85.00\pm0.59 &   67.07\pm2.47 &     50.93\pm 3.47\\
      \midrule                    
      N$k$F (ours)                       & 4K        &   78.77\pm0.55 &   80.30\pm2.42 &  64.42\pm2.09 & \b93.42\pm1.94 & \b67.69\pm1.28 &    62.93\pm4.13\\
      \bottomrule
    \end{tabular}}\end{table}

\subsection{Real-World Graphs}
In this experiment we attempt to use our model to leverage the geometry
of non-equivariant node features for \emph{graph classification} on
a set of benchmark datasets.
The basic architecture of the model follows that described in the
\emph{Architecture and Parameters} section. Graphs are represented as
a $1$-chain consisting of all their edges. We randomly initialise a set
of $\ell$ feature $1$-forms, produce and read-out the integration
matrices before feeding through a simple MLP classifier and performing
backpropagation. We use an $L_2$-column readout layer so the network is
invariant under edge orientations. 
Please refer to \cref{app:Graph Benchmark Datasets} for specific
architectural details.
We use state-of-the-art graph neural networks~(GNNs) following
a recently-described benchmark~\citep{Dwivedi23a}, experimenting with
different numbers of layers.
As an additional comparison partner, we also use a recent equivariant
GNN architecture that is specifically geared towards analysing data
with an underlying geometry~\citep{Satorras21a}.
\Cref{tab:Benchmark datasets small} depicts the results on smaller graph
datasets in the TU dataset \citep{Morris+2020}.
Here the node features that carry both equivariant
information~(corresponding to 3D coordinates of molecules, for instance)
and non-equivariant information (one-hot atomic type, weight, etc.). For
the non-equivariant models (ours, GCN, GAT) the position features are
omitted. 
Overall, our method exhibits competitive performance, in particular given the
fact that it does not make use of \emph{any} message passing and has a smaller
parameter footprint.
In \cref{tab:Benchmark datasets large} we compare our model on the
larger datasets in the `MoleculetNet' benchmark~\citep{Wu18a}. Note that
the datasets we have chosen have no provided 'positional' node features,
so the given node features (i.e.\ atomic weights, valence, etc.) are not
equivariant and cannot be compared with EGNN.
As \cref{tab:Benchmark datasets large} shows, our model based on neural
$k$-forms outperforms state-of-the-art graph neural networks in terms of
AUROC, using a fraction of the number of parameters~(this also holds for
accuracy and average precision, which are, however, not typically
reported for these data sets).

\begin{table}[tbp]
  \centering
  \caption{Results~(mean AUROC and standard deviation of $5$ runs)
    on benchmark datasets from the `MoleculeNet' database~\citep{Wu18a}.
While the GNNs also train for smaller numbers of parameters, we
    observed significant drops in predictive performance.
We thus report only the best results for GNNs, using the most common
    models described in the literature.
  }
  \label{tab:Benchmark datasets large}
  \small
  \sisetup{
    detect-all              = true,
    table-format            = 2.2(4),
    separate-uncertainty    = true,
    retain-zero-uncertainty = true,
  }\let\b\bfseries \begin{tabular}{lrSSS}
     \toprule
                                        & {Params.}      & {BACE}         & {BBBP}         & {HIV}\\
     \midrule
GAT~\citep{velickovic2018graph}    & 135K           &  69.52\pm17.52 &  76.51\pm 3.36 &  56.38\pm 4.41\\
     GCN~\citep{kipf2017semisupervised} & 133K           &  66.79\pm 1.56 &  73.77\pm 3.30 &  68.70\pm 1.67\\
     GIN~\citep{xu2018how}              & 282K           &  42.91\pm18.56 &  61.66\pm19.47 &  55.28\pm17.49\\
     \midrule
     N$k$F~(ours)                       & 9K             &\b83.50\pm0.55  &\b86.41\pm 3.64 &\b76.70\pm 2.17\\
     \bottomrule
  \end{tabular}\end{table}

\section{Discussion}

\paragraph{Summary.}
We developed \emph{neural $k$-forms}, a method for learning
representations of simplicial cochains to solve tasks on embedded graphs
and simplicial complexes.
Deviating from the predominant paradigms in \emph{geometric deep
learning}, our method adopts a fundamentally different novel perspective
based on integration of forms in the ambient feature space.
We have shown the feasibility of the method through a comprehensive
experimental suite, demonstrating its effectiveness using only a very
small number of parameters.
Notably, our method does not utilise any kind of message passing, and we
hypothesise that it is possible that this implies that issues like
\emph{over-smoothing} may affect our method less than graph neural
networks.

\paragraph{Limitations.} 
A conceptual limitation is that we require (at least) the existence of
node or vertex coordinates, i.e.\ our method only operates on embedded
complexes.
The computational feasibility of higher order $k$-forms in large
embedding spaces is another possible limitation. Indeed, the number of
monomial $k$-forms in $\R^n$ is $\binom{n}{k}$, and similar issues arise
for numerical integration over higher-dimensional simplices. 
Further, we have only benchmarked our method on graph classification
tasks.
It remains to be seen whether the method performs as well on
graph regression tasks, as well as for benchmark learning tasks on
higher-dimensional simplicial complexes.\footnote{We note that
a comprehensive, agreed-upon framework for benchmarking simplicial
neural networks has yet to be established.}
Finally, our model currently cannot deal with node features that are
equivariant with respect to the embedding space.

\paragraph{Outlook.}
We aim to study these issues, as well as the behaviour of our methods in
the context of long-range dependencies, in a follow-up work. 
In addition, since our neural $k$-form formulation is equivalent to an
MLP, the learning process may benefit from the plethora of existing
methods and tricks that are applied to optimise MLPs in practice.
We argue that our experiments point towards the utility of using
a geometric interpretation of our representations as \emph{integrals}
over $k$-forms may provide valuable insights to practitioners.
Lastly, we provide a small example of a \emph{Convolutional $1$-form
network} in \cref{sec:add-exp} that may lead to better incorporation of
equivariant node embeddings. We provide this auxiliary example as part
of a broader future work program of rebuilding common ML architectures
on top of neural $k$-forms rather than message passing schemes.

\subsubsection*{Acknowledgements}

B.R.\ is supported by the Bavarian state government with
funds from the \emph{Hightech Agenda Bavaria}.
K.M.\ received funding from the European Union’s Horizon 2020 research and innovation programme under the Marie Skłodowska-Curie grant agreement No 859860. 
The authors gratefully acknowledge the Leibniz Supercomputing Centre for
funding this project by providing computing time on its Linux cluster.
The authors also acknowledge Darrick Lee, Kathryn Hess, Julius von
Rohrscheidt, Katharina Limbeck, and Alexandros Keros for providing
useful feedback on an early manuscript.
They also wish to extend their thanks to the anonymous reviewers and the
area chair, who believed in the merits of our work.

\bibliography{biblio}
\bibliographystyle{iclr2024_conference}

\clearpage
\appendix

\section{Differential Forms in \texorpdfstring{$\R^n$}{Rn} Background} 

\label{diff_form_appendix}

In this section, we provide a basic background in the theory of
differential forms. The more complete references see \citet{Lee03} for
basic differential geometry and \citet{Jost_2002} for Riemannian
geometry of $k$-forms. Moreover, see \citet{Tao_2009} for a short but
highly illuminating article on the intuition behind differential forms.

\paragraph{Tangent and Cotangent Bundles on $\R^n$.} In $\R^n$, the \textit{tangent space} $T_p(\R^n)$ at a point $p \in \R^n$ is the vector space space spanned by the partial derivatives $\D /\D x_i(p)$. The \textit{cotangent space} at $p$ is the linear dual $T^*(\R^n)$ of the tangent space; i.e. the space of linear maps $\mathsf{Hom}_\R(T_p(\R^n),\R)$. Note that both spaces are isomorphic to $\R^n$. The \textit{tangent bundle} is the space $T(\R^n) = \sqcup_p T_p(\R^n)$ consisting of gluing all tangent space together, where the topology is induced by the projection map $\pi : T(\R^n) \to \R^n$. Likewise, the \textit{cotangent bundle} is the space $T^*(\R^n) = \sqcup_p T_p^*(\R^n)$. 

The space of vector fields $\mathfrak{X}(\R^n)$ over $\R^n$ is the space of sections of the tangent bundle; that is, maps $v : \R^n \to T_*(\R^n)$ such that $\pi \circ v = id_{\R^n}$. Vector fields decompose into the form $\sum_i v_i(p) \D/\D x_i(p)$, where $v_i : \R^n \to \R$. The space of $1$-forms $\Omega^1(\R^n)$ over $\R^n$ consists of the sections $\omega : \R^n \to T^*(\R^n)$ of the cotangent bundle. 

\paragraph{Riemannian Structure on $\R^n$.} The space $\R^n$ is a Riemannian manifold. That is, it has a non-degenerate bilinear form 
\begin{equation}
    \langle -, - \rangle_p : T_p(\R^n) \otimes T_p(\R^n) \to \R.
\end{equation} The inner product is defined on by linear extension of the formula \begin{equation} \langle \D/\D x_i(p), \D/\D x_j(p) \rangle_p = \langle x_i, x_j \rangle \end{equation} where the second inner product is the standard inner product on $\R^n$. Over $\R^n$, the spaces of vector fields and $1$-forms are isomorphic via the \textit{sharp isomorphism}
\begin{align} 
\# : \mathfrak{X}(\R^n) & \to \Omega^1(\R^n)\\
\sum_i v_i \D/\D x_i & \mapsto \sum_i v_i dx_i
\end{align} where $dx_i$ is the $1$-form which is locally $dx_i(p) = \langle -, \D/\D x_i(p) \rangle_p : T_p(\R^n) \to \R$.

\paragraph{Exterior Products.} Let $V$ be a real vector space. Recall that an alternating $k$-covector on $V$ is a map
\begin{equation} 
\alpha : \bigotimes_k V \to \R
\end{equation} that is alternating with respect to permutations.  That is, that
\begin{equation} 
\alpha(v_1, v_2, \ldots, v_k) = (-1)^{sgn(\tau)} \alpha(v_{\tau(1)}, v_{\tau
(2)}, \ldots, v_{\tau(k)}).
\end{equation} The $k$-th exterior product $\Lambda^k(V^*)$ over $V$  is the space of alternating $k$-covectors.

\paragraph{$k$-Forms.} The $k$-th exterior product of the cotangent bundle $\Lambda^k T^* (\R^n)$ is the tensor bundle defined by locally taking the $k$-th exterior product $\Lambda^k T^*_p (\R^n)$ of each cotangent space. A \textit{differential $k$-form} $\omega \in \Omega^k(\R^n)$ over a $\R^n$ is a smooth section of the $k$-th exterior product of the cotangent bundle $\Lambda^k T^* M$. The space of forms $\Omega^*(\R^n)$ of any dimension has an algebra structure given by the wedge product. The wedge product is multi-linear and satisfies anti-commutativity
\begin{equation} \label{wedge-anticommutativity}
    dx_i \wedge dx_j = - dx_j \wedge dx_i
\end{equation}
as well as permutation equivariance
\begin{equation} \label{wedge-perm-equiv}
    dx_{i_1} \wedge \ldots \wedge dx_{i_k} = (-1)^{sgn(\tau)} dx_{i_{\tau(1)}} \wedge \ldots \wedge dx_{i_{\tau(k)}}
\end{equation} for permutations $\tau$. As described in the body of the paper, $k$-forms have a canonical monomial decomposition given by
\begin{equation}
    \omega = \sum_I \alpha_I dx_I.
\end{equation} 
where $dx_I = dx_{i_1} \wedge \ldots \wedge dx_{i_k}$ for a multi-index $I = (i_1, \ldots, i_k)$ and scaling maps $\alpha_I : \R^n \to \R$.

\paragraph{Types of $k$-forms.} Differential $k$-forms are $k$-forms where the scaling maps $\alpha_I$ are smooth; this is equivalent to the condition $\omega$ that is a smooth section of $k$-th exterior product of the cotangent bundle. When working with both neural $k$-forms and over non-compact spaces like $\R^n$, we often need to define other types of forms. 
\begin{enumerate} 
\item A $k$-form $\omega \in \Omega^k_c(\R^n)$ is \textit{compactly supported} whenever each of the $\alpha_I$ are compactly supported.
\item A $k$-form $\omega \in L_2 \Omega^k(\R^n)$ is $L_2$ whenever each $\alpha_I$ is square integrable. 
\item A $k$-form $\omega \in \Omega^k_{PL}(\R^n)$ is \textit{piecewise linear} if there exists a triangulation of $\R^n$ such that each $\alpha_I$ is piece-wise smooth over the triangulation.
\end{enumerate}

\paragraph{Inner Products on $k$-forms.} The choice of Riemannian metric on $\R^n$  extends to an inner product on the $k$-th exterior product of the cotangent space by
\begin{align*} 
\langle -, - \rangle_p : \Lambda^k T^*_p(\R^n) \otimes \Lambda^k T^*_p(\R^n) & \to \R\\ 
\langle \omega_1 \wedge \ldots \wedge \omega_k, \eta_1 \wedge \ldots \wedge \eta_k\rangle_p & \mapsto \det \langle \omega_i, \eta_j \rangle_{i,j}.
\end{align*} where $\omega^i, \eta_j \in \Omega^1(\R^n)$. This induces an inner product over the compactly supported $k$-forms $\Omega_c^k(M)$ by integration
\begin{equation} 
\langle \omega_1 \wedge \ldots \wedge \omega_k, \eta_1 \wedge \ldots \wedge \eta_k \rangle = \int_{p \in \R^n} \langle \omega_1 \wedge \ldots \wedge \omega_k, \eta_1 \wedge \ldots \wedge \eta_k
\rangle_p \, d Vol(\R^n)
\end{equation} against the volume form $Vol(\R^n)$ on $\R^n$. 

\paragraph{Orthonormal Coframes.} Equip $\R^n$ with the usual Riemannian structure, the compactly supported $k$-forms $\Omega^k_c(\R^n)$ with the inner product $\langle -, -\rangle_\Omega$ described above and induced norm $$\lVert \omega \rVert^2_\Omega = \langle \omega, \omega \rangle_\Omega$$ for $\omega \in \Omega^k_c(\R^n)$. With this structure, the monomial $k$-forms form an orthornomal coframe in the sense that
\begin{equation} \label{orth} 
\langle dx^I, dx^{I'} \rangle_p = \begin{cases} 1 & \text{if } I = I' \\ 0 & \text{else} \end{cases} \end{equation} for each $p \in \R^n$.

\paragraph{Integration of $k$-forms.} Following \citep{taylor2006measure}, we define integration of monomial $k$-forms $\omega = \sum_I \alpha_I dx_I$ then extend via linearity to arbitary $k$-forms. Let $\phi : \Delta^k \to \R^n$ be a smooth map and coordinatize $\Delta^k$ with $(t_1, t_2, \ldots, t_k)$ as above.

The \textit{pullback} map of $\phi$ is a map $\phi^* : \Omega(\R^n)\to \Omega(\Delta^k)$ taking forms on $\R^n$ to forms on $\Delta^k$. In coordinates, the pullback $\phi^* \omega \in \Omega^k(\Delta^k)$ of $\omega$ along $\phi$ is defined by the formula
\begin{equation} \label{pullback}
    \phi^* \omega = \sum_I \alpha_I (\phi^* dx_{i_1}) \wedge (\phi^* dx_{i_2}) \wedge \ldots \wedge (\phi^* dx_{i_k})
\end{equation} where
\begin{equation}
    \phi^* dx_{i} = \sum_j \dfrac{\D \phi_i}{\D t_j} dt_j
\end{equation} and $\phi_i$ is the $x_i$-component of $\phi$. Note that by the monomial decomposition in \Cref{form_decomp}, the pullback can be written as
\begin{equation}
    \phi^* \omega = f dt_1 \wedge dt_2 \wedge \ldots \wedge dt_k
\end{equation} for some smooth function $f : \Delta^k \to \R$.

We define the integral $\int_\phi$ to be the standard Riemann integral
\begin{equation} \label{form-integral-def}
    \int_\phi \omega := \int_{\Delta^k} f dt_1 dt_2 \ldots dt_k.
\end{equation} 

The function $f$ can be computed explicitly by unwinding \ref{pullback} using the algebraic relations \ref{wedge-anticommutativity} and \ref{wedge-perm-equiv}. Namely, for the monomial $k$-form $\omega = \alpha_I dx_I$ we have
\begin{align} \phi^* \omega & = \alpha_I(\phi) \Big( \sum_j \dfrac{\D \phi_{i_1}}{\D t_j} dt_j \Big) \wedge \ldots \wedge \Big( \sum_j \dfrac{\D \phi_{i_k}}{\D t_j} dt_j \Big)\\
& = \alpha_I(\phi) \Bigg( \sum_\tau sgn(\tau) \dfrac{\D \phi_{i_1}}{\D t_{\tau(1)}} \cdots \dfrac{\D \phi_{i_k}}{\D t_{\tau(k)}} \Bigg) dt_1 \wedge \ldots \wedge dt_k\\
& = \alpha_I(\phi) \varepsilon^I( D\phi) dt_1 \wedge \ldots \wedge dt_k
\end{align} The above calculation in conjunction with \ref{form-integral-def} recovers the formula for the integral
\begin{equation}
    \int_\phi \omega = \sum_I \int_{\Delta^k} \alpha_I(\phi) \varepsilon^I(D \phi) dt
\end{equation} of a general $k$-form $\omega$ given in \ref{integration_eq}.

\begin{remark}
    In the special case that $\phi$ is an affine map, then the Jacobian is expressed as
    $$D\phi_{i,j} = \Big[ \phi^i(t_j) - \phi^i(t_0) \Big]_{i,j}$$
\end{remark}

\paragraph{Properties of Integration.} Throughout the paper, we refer to the linearity and orientation equivariance of integration of forms over simplices. These are a consequence of the following theorem.

\begin{proposition}[\citep{Lee03}, 16.21] Suppose $M$ is an $n$-manifold with corners and $\omega, \eta \in \Omega^n(M)$ are smooth $n$-forms. Then
\begin{enumerate}
    \item For $a,b \in \R$ we have
    $$ \int_M a\omega + b\eta = a \int_M \omega + b \int_M \eta$$
    \item Let $-M$ denote $M$ with the opposite orientation. Then
    $$\int_{-M} \omega = - \int_M \omega.$$
\end{enumerate}
    
\end{proposition}

\section{Integration of \texorpdfstring{$k$}{k}-forms in Practice}\label{sec:A-integration}
We give further practical details on how to approximate integrals in the case of $2$-forms. For more detailed and accessible introductions we recommend the following references \citep{taylor2006measure, Tao_2009}.

\paragraph{Explicit computations for $2$-forms.} In Section \ref{sec:experiments} we use integration of $2$-forms on $2$-dimensional simplicial complexes in order to classify surfaces. We will derive an explicit method to approximate the integral of $2$-forms in $\R^n$ over an embedded $2$-simplex based on the theory above. 

Recall that a $2$-form $\omega\in\Omega^2(\R^n)$ can be written in coordinates as 
\begin{equation}
    \omega = \sum_{0\leq i <j \leq n} \alpha_{i,j} dx_i\wedge dx_j,
\end{equation}
where $\alpha_{i,j}: \R^n\longrightarrow\R$ are smooth maps. We consider a map $\phi:\Delta^2 \longrightarrow \R^n$ giving the embedding of the standard $2$ simplex in $\R^n$. Using the expression from Equation \ref{form-integral-def}: 
$$\int_{\phi(\Delta^2)}\omega = \int_{\Delta^2}\phi^* (\omega)$$
we will show how to integrate $\omega$ on the embedded simplex $\phi(\Delta^2)$.

Using local coordinates to express the embedding $\phi(t_1, t_2) = (x_1, \dots ,x_n)$ we can write out explicitly the pullback of a $2$-form $\omega\in \Omega^2(\R^n)$ via the map $\phi$: 
\begin{align*}
    \phi^*\big(\sum_{0\leq i <j \leq n} \alpha_{i,j} dx_i\wedge dx_j \big) 
    = \sum_{0\leq i <j \leq n}\big(\alpha_{i,j}\circ \phi \big) \big( \frac{\partial x_i}{\partial t_1} dt_1 +  \frac{\partial x_i}{\partial t_2} dt_2 \big)\wedge\big( \frac{\partial x_j}{\partial y_1} dt_1 +  \frac{\partial x_j}{\partial t_2} dt_2 \big) \\
   = \sum_{0\leq i <j \leq n}\big(\alpha_{i,j}\circ \phi \big) \Big( \big( \frac{\partial x_i}{\partial t_1} dt_1 \wedge  \frac{\partial x_j}{\partial t_2} dt_2 \big) + \big( \frac{\partial x_i}{\partial t_2} dt_2 \wedge  \frac{\partial x_j}{\partial t_1} dt_1  \big) \Big)
\end{align*}
\begin{align}\label{eq:A}
     = \sum_{0\leq i <j \leq n}\big(\alpha_{i,j}\circ \phi \big) \Big( \underbrace{\frac{\partial x_i}{\partial t_1}\frac{\partial x_j}{\partial t_2} -  \frac{\partial x_i}{\partial t_2}\frac{\partial x_j}{\partial t_1}}_{A}\Big) dt_1\wedge dt_2 
\end{align}

In the computational context that we consider we want to integrate a $2$-form on the $2$-simplices of a complex embedded in $\R^n$. We set some notation here, each $2$-simplex $\sigma\in X$ is the image of an affine map $ \phi (t)= \Phi t + b : \Delta^2 \longrightarrow \R^n$, where $\Phi = D \phi = [\phi_{i,1}, \phi_{i,2}]_{0 \leq i \leq n }$ is a $(n\times 2)$-matrix. Since $\phi$ is an affine map the term $A$ in \ref{eq:A} becomes
    \begin{equation}
        A = \varepsilon^{(i,j)}( D \phi) = \det \Phi_i^j, 
    \end{equation}
    with $\Phi_i^j$ denoting the $(2\times 2)-$submatrix of $\Phi$ corresponding to the rows $i$ and $j$. 

Putting all this together we obtain: 
\begin{equation}
    \phi^*(\omega) = \sum_{0\leq i<j\leq n } \alpha_{i,j}(\phi) \det \Phi_i^j dt_1\wedge dt_2.
\end{equation}
So the integration becomes 

\begin{equation}
    \int_{\phi(\Delta^2)} \omega = \int_{\Delta^2}  \phi^*(\omega) = \int_{\Delta^2} \underbrace{\sum_{0\leq i<j\leq n } \alpha_{i,j}(\phi) \det \Phi_i^j}_{g} dt_1\wedge dt_2 .
\end{equation}

Computationally we cannot perform exact integration of $2$-forms so in practise we approximate the integral on $\Delta^2$ above by a finite sum, as is done with Riemann sums in the classical case. In general this is how the $\texttt{VolApprox}$ function is defined. To do so the first step is to subdivide the domain of integration $\Delta^2$ into a collection $\mathcal{S}$ of smaller simplices with vertices denoted by $v$. Then the integral is approximated by summing the average value of the map $g$ on each simplex of the subdivision: 
\begin{equation}
    \int_{\Delta^2} g dt_1\wedge dt_2 \approx \sum_{s\in \mathcal{S}^2} \frac{1}{3}\sum_{v\in s} g(v)\cdot vol(s).
\end{equation}

An explicit method of this sort can be easily generalized to higher dimensional simplices in order to compute integrals of $k$-forms on $k$-dimensional simplicial complexes.

\section{Proofs} \label{proof-appendix}

 \paragraph{Universal Approximation I.} We start with the proof of Proposition \ref{prop:univ-approx-1}. First, recall the well-known following universal approximation theorem for neural networks. We cite here a version appearing in \citep{Pinkus99}.

\begin{theorem}[Universal Approximation Theorem, Thm 3.1 \citep{Pinkus99}] \label{universal_function_approx}
    Let $\sigma \in C(\R, \R)$ be a non-polynomial activation function. For every $n,\ell \in \N$, compact subset $K \subset \R^n$, function $f\in C(K, \R^\ell)$ and $\epsilon > 0$ there exists:
    \begin{itemize}
        \item an integer $j \in \N$;
        \item a matrix $W_1 \in \R^{j \times n}$;
        \item a bias vector $b \in \R^{j}$ and
        \item a matrix $W_2 \in \R^{\ell \times j}$
    \end{itemize}
    such that 
    $$ \lVert f - g \rVert_\infty < \varepsilon$$
    where $g$ is the single hidden layer MLP
    $$g(x) = W_2 \sigma(W_1 x + b).$$
\end{theorem}

\begin{reptheorem}{prop:univ-approx-1}
    Let $\alpha \in C(\R, \R)$ be a non-polynomial activation function. For every $n \in \N$ and compactly supported $k$-form $\eta \in \Omega^k_c(\R^n)$ and $\epsilon > 0$ there exists a neural $k$-form $\omega^\psi$ with one hidden layer such that 
    $$ \lVert \omega^\psi - \eta \rVert_\Omega < \epsilon.$$
\end{reptheorem}

\begin{proof}
    Denote the scaling functions of $\eta$ by
    $$ \eta = \sum \eta_I dx_I.$$ 
    Since $\eta$ is compactly supported, each $\eta_I$ is compactly supported over some domain $D$. By \ref{universal_function_approx}, for any $\ep > 0$ there exists a one layer MLP $\psi : \R^n \to \R^{\binom{n}{k}}$ such that
    $$\lVert \psi - \bigoplus_I \eta_I \rVert_\infty < \ep / Vol(D)^{1/2}.$$ Using the orthogonality form $\Cref{orth}$, we have that
    $$\lVert \omega^\psi - \eta \rVert^2_p = \langle  \omega^\psi - \eta,  \omega^\psi - \eta \rangle_p =  \sum_I (\psi_I(p) - \eta_I(p) )^2 = \lVert \psi(p) - \oplus_I \eta_I(p) \rVert_{\R^{\binom{n}{k}}}^2 \leq \ep^2/Vol(D).$$ Integrating the above, we attain
    \begin{align*} \lVert \omega^\psi - \eta \rVert^2_\Omega & \leq \int_D \langle \omega^\psi - \eta, \omega^\psi - \eta \rangle_p dVol(\R^n)\\
    & < \ep^2 / Vol(D) \int_D dVol(\R^n) \\
    &  = \ep^2
    \end{align*}
    proving the result.
\end{proof}

\paragraph{Equivariance.} Integration is linear in both $k$-forms (\cite{Lee03}, 16.21) and embedded simplicial chains \ref{int-embedded-simplicial} in the sense that it defines a bilinear pairing
\begin{equation}
    \int : C_k(\mathcal{S};\R) \otimes \Omega^k(\R^n) \to \R
\end{equation} for some embedded simplicial complex $\phi : \cS \to \R^n$. This property directly implies the kind of multi-linearity described in \Cref{multi-linearity}. We present a proof here for completeness.

\begin{repproposition}{multi-linearity}[Multi-linearity] Let $\phi: \mathcal{S} \to \R^n$ be an embedded simplicial complex. For any matrices $L \in M^{m' \times m}(\R)$ and $R \in M^{\ell \times \ell'}(\R)$ we have
    \begin{equation}
        X_\phi(L \beta, \omega R) = L X_\phi(\beta, \omega) R
    \end{equation}
\end{repproposition}

\begin{proof}
    On the left we have
    \begin{align}
        X_\phi(L \beta, \omega)_{i,j} & = X\Big( \sum_k L_{1,k} \beta_k, \ldots, \sum_k L_{m',k} \beta_k, \omega\Big)_{i,j}\\
        & = \int_{\sum_k L_{i,k} \beta_k} \omega_j\\
        & = \sum_{k} L_{i,k} \int_{\beta_k} \omega_j\\
        & = \Big[ L X_\phi(\beta,\omega) \Big]_{i,j}.
    \end{align}
    Similarly, on the right we have
    \begin{align}
        X_\phi(\beta,\omega R)_{i,j} & = X_\phi(\beta, \sum_k R_{k,1} \omega_k, \ldots, \sum_k R_{k,\ell'} \omega_k\Big)_{i,j}\\
        & = \int_{\beta_i} \sum R_{k,j} \omega_k\\
        & = \sum_k R_{k,j} \int_{\beta_i} \omega_k\\
        & = \Big[ X_{\phi}(\beta,\omega) R \Big]_{i,j} 
    \end{align}
\end{proof}

\section{Additional Experiments} \label{sec:add-exp}

\subsection{Convolutional $1$-form Networks} 

 \paragraph{Convolution and Equivariance.} Beyond orientation and permuation equivariance, the node embeddings themselves may be equivariant with respect to a group action on the feature space. The class of \textit{geometric graph neural networks} \cite{han2022} were designed specifically to deal with such equivariances. The lesson from image processing is that translation equivariance can be mitigated via \textit{convolution}. In this paradigm processing embedded (oriented) graphs with learnable, `template' $1$-forms  is directly analogous to processing images with learnable convolutional filters. To make this connection precise, we will perform a small example to classify oriented graphs.

\paragraph{Synthetic Data.} In Figure \ref{fig:conv-net}, there are two classes of cycle and star graphs. The cycle graphs have clockwise orientation and the star graph are oriented so edges point inwards. Each data-point is a simplicial complex representing the disjoint union of either three cycle or star graphs which are randomly recentered around the unit square. The chains on each complex are initialised using the standard oriented $1$-simplices as a basis. 

\paragraph{Architecture.} We initialise a neural network $ \psi : \R^2 \to \R^{2\times 2}$ with a single 64-dimensional hidden layer and \texttt{ReLU} activation which corresponds to two feature $1$-forms. To perform a convolutional pass, we first discretize the unit square to produce a set of translations. At each translation, we restrict the embedded graph to the subgraph within a neighbourhood and weight it by the local node density approximated by a standard kernel density estimator. This is equivalent to translating the $1$-form by the corresponding (negative) vector in $\R^2$ and integrating over a small neighbourhood (whose area is a hyperparameter). Integration produces an integration matrix at each point with two columns, and taking the column sum represents the oriented integral of the two convolutional filters against the local neighbourhood of the oriented graph. \texttt{CrossEntropyLoss} is calculated by summing the integrals over all translations and applying \texttt{softmax}. 

\paragraph{Interpreting the Results.} The integrals are shown in Figure \ref{fig:conv-net} as a colouring of each point in the grid of translations. We plot the learned 'template' 1-forms next to an example of their respective classes. By the construction of the objective function, the algorithm is trying to maximise the sum of the integrals across all translations within the class. As in the synthetic paths example, the learned filters successfully capture the interpretable, locally relevant structure; the edges in each class resemble flow-lines of the learned vector field in different local neighbourhoods around the relevant translations.

\begin{figure}
    \centering
    \includegraphics[width = 0.8\linewidth]{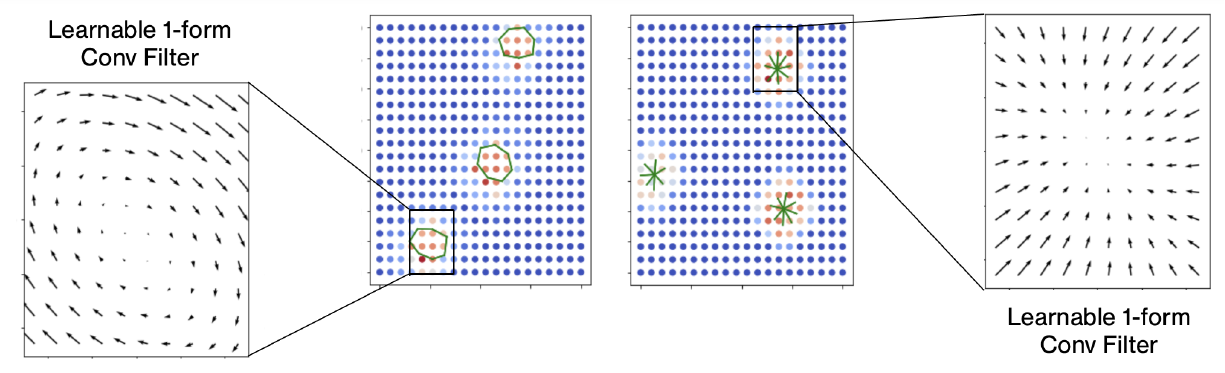}
    \caption{ \centering  A convolutional $1$-form network with the learned convolutional filters forms.}
    \label{fig:conv-net}
\end{figure}

\subsection{Visualising Simplicial Laplacian \texorpdfstring{$1$}{1}-Eigenvectors}

In this example, we show how we can use neural $1$-form to visualise the $1$-eigenvectors of the simplicial Laplacian\footnote{See Appendix for definitions.} of a Rips complex  in $\R^2$ (\Cref{fig:1-eigenvectors-rips1}). The idea is that we start with a collection of eigenvectors of simplicial $1$-cochains, and learn $1$-forms which integrate to them. We see the results in \Cref{fig:learned-1-eigs1}, where the three eigenvectors belong to the harmonic, gradient-like and curl components of the simplicial Hodge decomposition respectively. 

First we generate a point cloud in $\R^2$ as a noisey approximation of a circle, then take the Rips complex for a fixed parameter $\varepsilon$. We then calculate the first $9$ eigenvectors of the $1$-simplicial Laplacian, sorted by increasing eigenvalue, and store them as columns of a matrix $Y$ with respect to the standard basis $\beta$. 
We initialise $9$ neural $1$-forms $$\omega = (\omega_1, \omega_2, \ldots, \omega_9) \in \bigoplus_\ell \Omega^1_{PL}(\R^2)$$ on $\R^2$ using a ReLU activation function. Our loss function is then
$$\mathcal{L}(\omega) = \lVert X_\phi(\beta,\omega) - Y \rVert$$ where the norm is the standard matrix norm, and $X_\phi(\beta, \omega)$ is the evaluation matrix yielded from integrating $\omega$ over $\beta$. When $\mathcal{L}$ is small, the $1$-forms over $\R^2$ will correspond to the selected simplicial eigenvectors when integrated over the complex. 

Figure \ref{fig:1-eigenvectors-rips1} shows the results---one notes that the different $1$-forms appear to coalesce around geometric features in the underlying point cloud. Additionally, it is important to observe that the Laplacian $1$-eigenvectors can be categorized into three distinct classes depending on their membership within the components of the Hodge decomposition:
\begin{enumerate}
    \item The first class consists of eigenvectors that belong to the image of the adjoint of the differential operator $d^1$,
    \item The second class comprises eigenvectors that reside within the kernel of the Laplacian operator $\Delta^1$, 
    \item Lastly, the third class includes eigenvectors that are part of the image of the differential operator $d^0$.
\end{enumerate}  These classes correspond, in the context of simplicial structures, to analogues of rotational, harmonic, and gradient-like vector fields found in Riemannian manifolds. By referring to \cref{fig:learned-1-eigs1}, one can visually identify the class to which each eigenvector belongs based on whether or not there exists rotational structure. We have color-coded these as green, red, and blue, respectively.

\begin{figure}
\centering
\begin{subfigure}{0.25\textwidth}
\centering
    {\includegraphics[width=\textwidth]{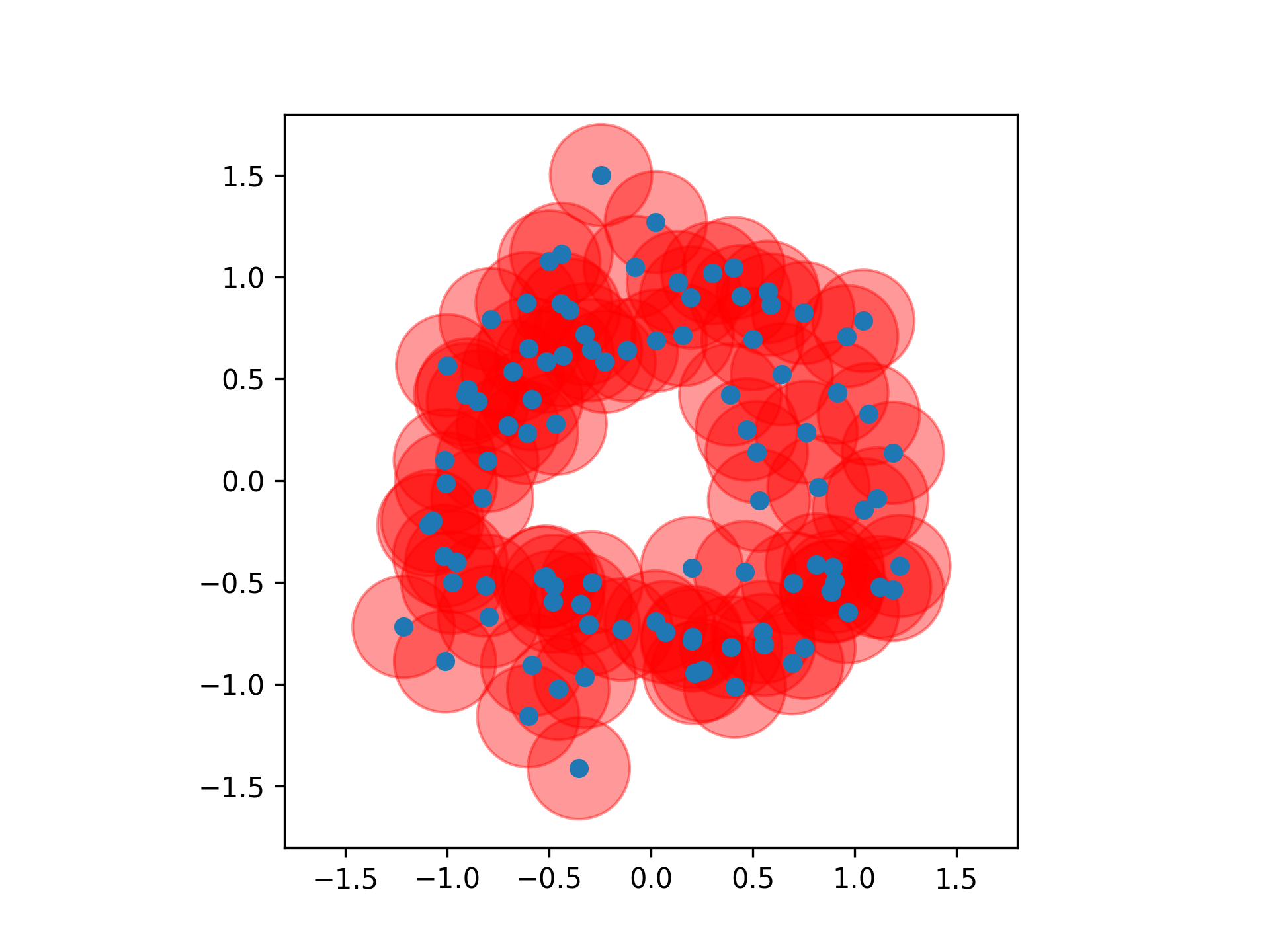}}
    \caption{\v{C}ech Cover}
    \label{fig:rips1}
\end{subfigure}
\begin{subfigure}{0.64\textwidth}
\centering
    \includegraphics[width=\textwidth]{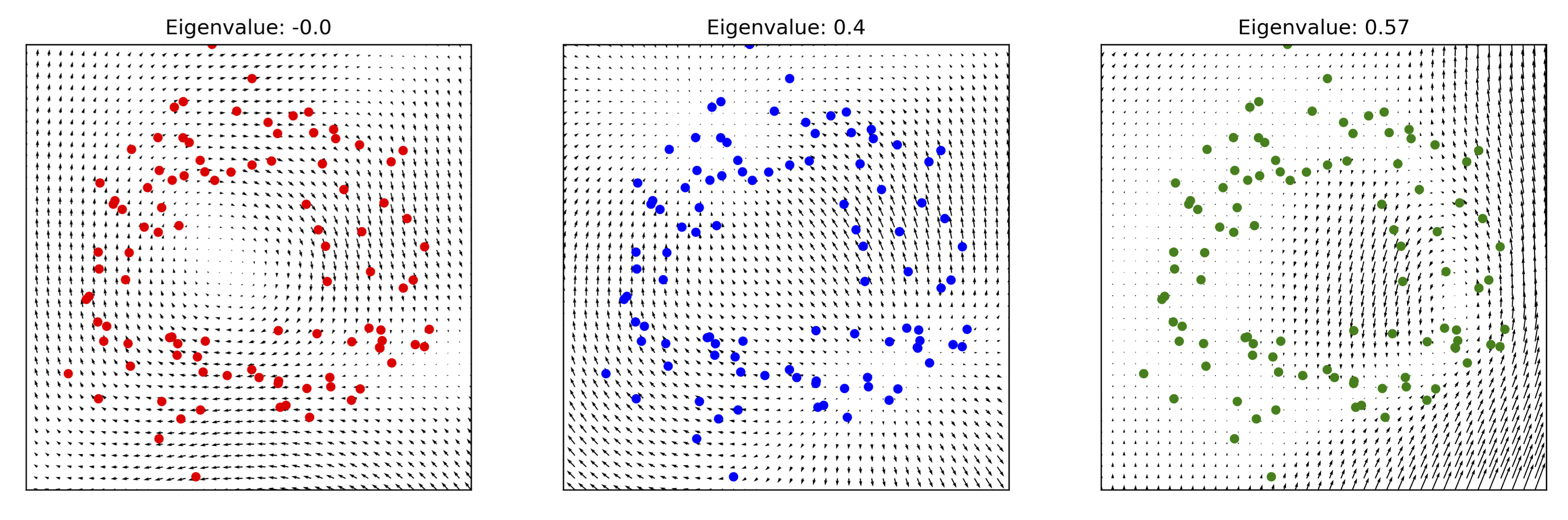}
    \caption{Learned $1$-forms}
    \label{fig:learned-1-eigs1}
\end{subfigure}
\caption{Approximating the $1$-eigenvectors of the simplicial Laplacian with learnable $1$-forms using the \v{C}ech Cover of a point cloud.
}
\label{fig:1-eigenvectors-rips1}
\end{figure}

\section{Experimental Setup and Implementation Details} \label{app:exp-details}

\subsection{Synthetic Path Classification}

\paragraph{Synthetic Path Classification.} We initialise neural $1$-forms $\omega_1, \omega_2, \omega_3 \in \Omega^1_{PL}(\R^2)$ with ReLU activation for each of the three classes. Each path $p$ is represented as an oriented simplicial complex, where the orientation is induced by the direction of the path. Letting $\beta$ be the standard basis, the three $1$-forms generate an evaluation matrix $X_\phi(\beta,\omega)$ whose entries are the integration of the $1$-form $\int_{e_i} \omega_j$ against each $1$-simplex in the paths. 

The readout function sums the entries in each column, which by the linearity of integration, represents the path integral along $p$ against each $\omega_i$. In short, each path $p$ is represented as a vector in $\R^3$ using the map
\begin{equation} \label{path-vectorization} p \mapsto (\int_p \omega_1, \int_p \omega_2, \int_p \omega_3) \in \R^3.\end{equation} We then use cross-entropy loss between this vector and the class vector. In this design, each $1$-form corresponds to a class, and a path should ideally have a high path integral against this form if and only if it belongs to the class.

\textit{Comparison with Neural Networks.} In our synthetic experiment, we test whether edge data is necessary by examining whether we can attain comparable results using only the embedded vertices of the path. 

In our context, the right framework to analyse the vertices is a set of $0$-forms on $\R^2$ - in other words, scalar functions over $\R^2$. We use the same setup as before, where we initialise three functions $f_1, f_2, f_3 \in \Omega_{PL}^0(\R^2)$ corresponding to the three classes. Integration of the vertices in the path against each of $f_i$ simply corresponds to evaluating $f_i$ at the vertex. Summing the columns of the evaluation matrix then sums up the value of $f_i$ at all points in a path. In this sense, each $f_i$ functions much like an approximated 'density' for the vertices of each class - albeit with negative values. 

In Figure \ref{fig:scalar-fields}, we show example paths from each class against the learned scalar function representing that class. In this example, the vertices of paths in each class have a similar density. One sees that the algorithm learns something reasonable, picking out minor fluctuations in density, but struggles overall to separate out the classes. With the same number of parameters and training time, the algorithm is only able to achieve a training accuracy of 37\%.

\begin{figure}
\centering
\includegraphics[width=0.9\textwidth]{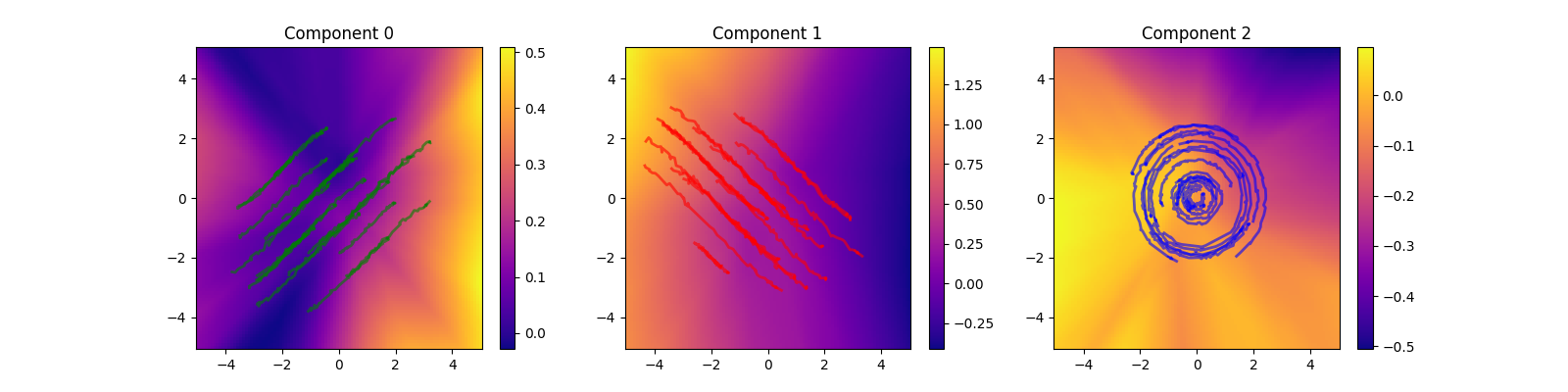}
\caption{Learned $0$-forms for Synthetic Path Classifcation.}
\label{fig:scalar-fields}
\end{figure}

\subsection{Synthetic Surface Classification}\label{sec:surface_ap}

We consider two classes of synthetic surfaces obtained by embedding a triangulated square in $\R^3$.
The embedding of the vertices of the triangulated square are given by functions of the type $\phi_1(x,y) = sin(x)+\epsilon(x,y)$ for the first class and $\phi_2(x,y) = sin(y)+\epsilon(x,y)$ for the second class, where $\epsilon$ is random noise. 
We initialise two neural $2$-forms $\omega_1, \omega_2 \in \Omega^2_{PL}(\R^3)$ with ReLU activation for each of the classes. 
Letting $\beta$ be the standard basis, the two $2$-forms generate an evaluation matrix $X_\phi(\beta,\omega)$ whose entries are the integration of the $2$-form $\int_{e_i} \omega_j$ against each $2$-simplex in the surfaces.
The readout function sums the entries in each column, which represents the integral of each $\omega_i$ over the entire surface, thus yielding a representation of the surfaces in $\R^2$ in the following way
\begin{equation} s  \mapsto (\int_s \omega_1, \int_s \omega_2,) \in \R^2.\end{equation} 
Finally, we use the cross entropy loss function to classify the surfaces into the two classes. In this design, each $2$-form corresponds to a class, and a surface should ideally have a high integral against this form if and only if it belongs to the class.

Figure \ref{fig:surfaces} shows the integral of $\omega_1$ on surfaces taken from each of the two classes. This integral is positive for elements of the first class and negative for elements of the second class. Similarly the values for the integral of $\omega_2$ is negative on surfaces of the first class and positive on elements of the second class.

\begin{figure}
\centering
\begin{subfigure}{0.45\textwidth}
\centering
    {\includegraphics[width=\textwidth]{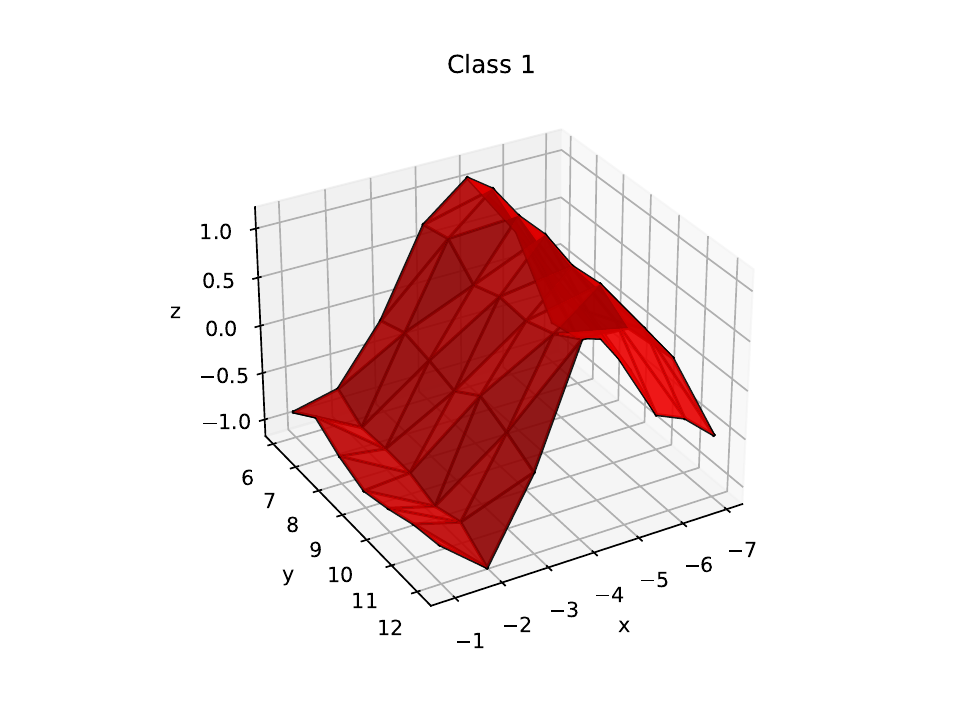}}
    \caption{Integral of $\omega_1$ on a representative of the first class.}
    \label{fig:surface1}
\end{subfigure}
\begin{subfigure}{0.45\textwidth}
\centering
    \includegraphics[width=\textwidth]{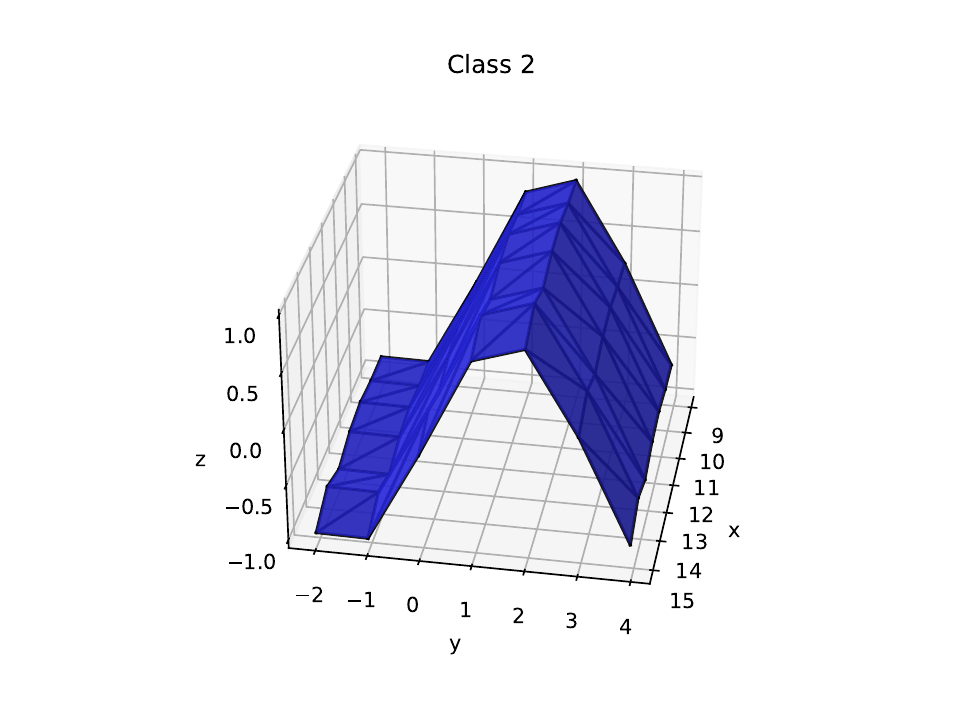}
    \caption{Integral of $\omega_1$ on a representative of the first class.}
    \label{fig:surface2}
\end{subfigure}
\caption{Representative surfaces of each classes, colored by the integral of the learned $2$-form $\omega_1$ over the surface. The integral of $\omega_1$ over elements of the first class yields positive values whereas the integral of the same form over elements of the second class yields negative values. For the form $\omega_2$ the opposite is true. 
}
\label{fig:surfaces}
\end{figure}

\subsection{Graph Benchmark Datasets}\label{app:Graph Benchmark Datasets}

For the small graph benchmark datasets~(AIDS, BZR, COX2, DHFR, Letter-low,
Letter-med, Letter-high), we use a learning rate of $1e-3$, a batch size
of $16$, a hidden dimension of $16$, and $h = 5$ discretisation steps
for all $k$-forms. For our comparison partners, i.e.\ for the graph
neural networks, we use~$h$ hidden layers to permit message passing,
followed by additive pooling. As a result, all models have roughly
equivalent parameter budgets, with our model having access to the
\emph{smallest} number of parameters.

\paragraph{Architectures.}
Our model architectures for our comparison partners follow the
implementations described in the respective
papers~\citep{kipf2017semisupervised, velickovic2018graph, xu2018how}. We make use of the
\texttt{pytorch-geometric} package and use the classes \texttt{GAT}, \texttt{GCN},
and \texttt{GIN}, respectively.
Our own model consists of a \emph{learnable vector field} and
a \emph{classifier} network. Letting $H$ refer to the hidden dimension
and $D_{\text{in}}$, $D_{\text{out}}$ to the input/output dimension,
respectively, we realise the vector field as an MLP of the form
\texttt{Linear[$D_{\text{in}}$,$H$] - ReLU - Linear[$H$,$H / 2$] - ReLU
- Linear[$H/ 2$,$D_{\text{out}}$]}. The \emph{classifier network}
consists of another MLP, making use the number of steps~$h$ for the
discretisation of our cochains. It has an architecture of
\texttt{Linear[$h$,$H$] - ReLU - Linear[$H$,$H / 2$] - ReLU
- Linear[$H/ 2$,$c$]}, where $c$ refers to the number of classes.

\paragraph{Training.}
We train all models in the same framework, allocating at most $100$
epochs for the training. We also add \emph{early stopping} based on the
validation loss with a patience of $40$ epochs. Moreover, we use
a learning rate scheduler to reduce the learning rate upon a plateau.

\end{document}